\documentclass{article}

\usepackage[preprint]{neurips_2025}
\usepackage[utf8]{inputenc} % allow utf-8 input
\usepackage[T1]{fontenc}    % use 8-bit T1 fonts
\usepackage{hyperref}       % hyperlinks
\usepackage{url}            % simple URL typesetting
\usepackage{booktabs}       % professional-quality tables

\usepackage{nicefrac}       % compact symbols for 1/2, etc.
\usepackage{microtype}      % microtypography
\usepackage{amsmath, amssymb, amsfonts} % blackboard math symbols
\usepackage[table]{xcolor} 
\usepackage{tcolorbox}
\definecolor{mdlblue}{RGB}{0,102,204}
\definecolor{mdlgreen}{RGB}{0,153,51}
\definecolor{mdlred}{RGB}{204,0,0}% colors
\usepackage{amsmath,amssymb,amsthm,amsfonts,enumitem}
\usepackage{graphicx}
\usepackage{animate}
\usepackage{tikz}
\usepackage{pgfplots}
\usepackage{multirow}
\pgfplotsset{compat=newest}

\usepackage{pgfplotstable}     % For tabular data, if needed
\usepackage{amsmath}           % For math formatting
\usepackage{xcolor}            % For color definitions

\usepackage{tcolorbox}         % For shaded boxes (optional)
\newtheorem{theorem}{Theorem}
\newtheorem{definition}{Definition}
\newtheorem{lemma}{Lemma}
\newtheorem{remark}{Remark}

\newtheorem{corollary}{Corollary}
\newcommand{\R}{\mathbb{R}}
\newcommand{\norm}[1]{\left\lVert #1 \right\rVert}
\newcommand{\grad}{\nabla}
\newcommand{\ip}[2]{\left\langle #1,\, #2 \right\rangle}  % inner product
\title{Bridging Predictive Coding and MDL: A Two-Part Code Framework for Deep Learning}

\author{%
  Benjamin Prada \\
  Bellini College of AI \\
  Cybersecurity and Computing \\
  University of South Florida \\ 
  Tampa, FL 33617 \\
  \texttt{bprada@usf.edu} \\
  \And
  Shion Matsumoto \\
  Bellini College of AI \\
  Cybersecurity and Computing \\
  University of South Florida \\
  Tampa, FL 33617 \\
  \texttt{matsumoto@usf.edu} \\
  \AND
  Abdul Malik Zekri\\
  Bellini College of AI \\
  Cybersecurity and Computing \\
  University of South Florida \\
  Tampa, FL 33617 \\
  \texttt{zekri2@usf.edu} \\
  \And
  Ankur Mali \\
  Bellini College of AI \\
  Cybersecurity and Computing \\
  University of South Florida \\
  Tampa, FL 33617 \\
  \texttt{ankurarjunmali@usf.edu} \\
}

\begin{document}

\maketitle
\begin{abstract}
We present the first theoretical framework that connects predictive coding (PC), a biologically inspired local learning rule, with the minimum description length (MDL) principle in deep networks. We prove that layerwise PC performs block-coordinate descent on the MDL two-part code objective, thereby jointly minimizing empirical risk and model complexity. Using Bernstein’s inequality and a prefix-code prior, we derive a novel generalization bound of the form \( R(\theta) \le \hat{R}(\theta) + \frac{L(\theta)}{N} \), capturing the tradeoff between fit and compression. We further prove that each PC sweep monotonically decreases the empirical two-part codelength, yielding tighter high-probability risk bounds than unconstrained gradient descent. Finally, we show that repeated PC updates converge to a block-coordinate stationary point, providing an approximate MDL-optimal solution. To our knowledge, this is the first result offering formal generalization and convergence guarantees for PC-trained deep models, positioning PC as a theoretically grounded and biologically plausible alternative to backpropagation.
\end{abstract}

\section{Introduction}

Over the past decade, deep learning has achieved remarkable empirical success across a wide array of applications \cite{krizhevsky2012imagenet, he2016deep, vaswani2017attention, devlin2019bert, hinton2012deep, radford2021learning, openai2023gpt4}. Yet, a foundational question remains unresolved: why do highly overparameterized neural networks generalize well, despite lacking explicit regularization or capacity control? While backpropagation (BP) remains the dominant training algorithm, it optimizes purely for loss minimization and offers no direct mechanism for balancing model complexity with data fit. This has motivated the exploration of alternative learning frameworks that offer both theoretical insight and practical efficiency. One principled approach to understanding generalization is rooted in information theory. The \textit{Minimum Description Length} (MDL) principle formalizes model selection as a coding problem: the best model is the one that minimizes the total length of a two-part code, with one part encoding the model and the other encoding the data given the model. Though MDL provides a compelling lens on generalization, it is rarely used to guide or analyze the training dynamics of modern deep networks—particularly those trained via biologically inspired algorithms.

In parallel, \textit{predictive coding} (PC) has re-emerged as a compelling alternative to BP. Originally proposed as a model of hierarchical processing in the neocortex \cite{rao1999predictive, friston2009predictive}, PC minimizes a layerwise energy functional by iteratively reducing prediction errors through local feedback and feedforward interactions \cite{salvatori2023brain}. Its simple, biologically plausible dynamics have enabled effective applications in image classification and generation \cite{ororbia2022neural, pinchetti2024benchmarking}, continual learning \cite{ororbia2020continual}, associative memory \cite{salvatori2021associative, tang2022recurrent}, and reinforcement learning \cite{rao2023active, ororbia2022active}. PC has also been shown to scale effectively to deep networks and neuromorphic hardware implementations \cite{kendall2020training, ororbia2019spiking}.

\textbf{From Dynamics to Generalization: An Open Gap.} 
Recent theoretical work has clarified many aspects of PC's learning dynamics, particularly in relation to BP. Milledge et al.\ \cite{millidge2023a} show that PC networks trained via prospective configuration converge to BP-critical points and approximate target propagation under certain inference equilibria. Mali et al.\ \cite{mali2024tight} approach PC through dynamical systems theory, proving Lyapunov stability, robustness to perturbations, and quasi-Newton behavior via higher-order curvature analysis. Other works have shown how PC approximates BP under assumptions like weight symmetry or linearization \cite{alonso2022theoretical, song2022inferring, Song2020, salvatori2022reverse, millidge2021predictive, millidge2022predictive}. While these studies provide valuable insight into the stability and efficiency of PC, they do not offer \textit{independent generalization guarantees}—that is, guarantees that do not rely on BP approximation. Moreover, none of these works establishes a connection between PC’s energy-based updates and formal learning-theoretic principles such as MDL or PAC-Bayes. Developing such guarantees is critical for explaining why PC-based models have recently demonstrated strong performance in few-shot reconstruction \cite{ ororbia2022convolutional, ororbia2020continual}, adversarial robustness \cite{salvatori2021associative}, and generalization under limited supervision \cite{ ororbia2022lifelong, salvatori2022learning}.

\textbf{In this paper, we close this gap by providing the first theoretical connection between predictive coding and the MDL principle.} Rather than treating PC as a proxy for BP, we reinterpret it as a standalone learning algorithm that implicitly minimizes a two-part empirical codelength objective. This reframing provides a new lens on PC’s ability to regularize model complexity and improve generalization, independent of its similarity to BP.

\noindent\textbf{Summary of Contributions.} The following key results provide the first principled MDL-based justification for PC:

\begin{tcolorbox}[
    colback   = blue!4!white,
    colframe  = mdlblue,
    title     = \textbf{Core Contributions (Summary)},
    fonttitle = \bfseries]
\begin{itemize}[leftmargin=*]

%---------------------------------------------------------------%
\item[\textcolor{mdlblue}{$\star$}]
\textbf{Universal Occam--style Bound (Algorithm–Agnostic).}
For every parameter realization $\theta$, sample size $N$, and confidence
$\delta\!\in\!(0,1)$ we have
\[
  \boxed{\textcolor{mdlblue}{
    R(\theta)\;\le\;
    \hat R(\theta)\;+\;\frac{L(\theta)}{N}\;+\;\frac{\ln(1/\delta)}{N}}}
  \quad\text{with}\quad
  L(\theta)= -\log p(\theta).
\]
This inequality holds \emph{independently} of whether the learner uses
BP, PC, or any other training rule.

%---------------------------------------------------------------%
\item[\textcolor{mdlgreen}{$\star$}]
\textbf{PC Minimizes Empirical Codelength:} Each PC sweep performs blockwise descent on the total description length:
\[
  \boxed{\textcolor{mdlgreen}{\hat{C}(\theta) = \hat{R}(\theta) + \frac{1}{N}L(\theta)}}
\]

%---------------------------------------------------------------%
\item[\textcolor{mdlred}{$\star$}]
\textbf{Convergence Guarantee:} We prove that repeated PC updates yield a bounded, monotonically decreasing sequence of codelengths that converges (up to subsequences) to a blockwise stationary point of \( \hat{C}(\theta) \). This provides the first convergence result for PC framed in terms of information-theoretic optimality.
\end{itemize}
\end{tcolorbox}

These results position PC not merely as a biologically inspired approximation to BP, but as a distinct, theoretically grounded learning algorithm—one that achieves generalization by minimizing description length. Our work bridges two historically disconnected perspectives—energy-based local learning and MDL-based generalization—and opens new directions for designing efficient, compression-driven learning systems beyond gradient BP.

\section{Background and Motivation}

Denote a model by $\theta \in \Theta$ (e.g., the parameters of a neural network).

\noindent\textbf{Minimum Description Length (MDL).}
The MDL principle \cite{Rissanen1978} posits that the best model is the one that most compresses the data. This is formalized as:
\[
C(D) = \min_{\theta \in \Theta} \left[ C(D \mid \theta) + C(\theta) \right],
\]
where \( C(D \mid \theta) \) is the data fit under the model and \( C(\theta) \) is the model complexity.

\vspace{0.8em}
\noindent\textbf{Predictive Coding Networks (PCNs).}
PC \cite{rao1999predictive, friston2009predictive} models cortical computation as hierarchical error minimization. PCNs consist of layers that predict the activity of the next and refine internal states through local feedback.

For nonlinear activations \( f \), PC minimizes the energy:
\[
\mathcal{F}(\theta, x) = \sum_{l=1}^L \left\| x_l - f(\theta_l x_{l-1}) \right\|^2,
\]
with \( x_0 \) clamped to the input and \( x_L \) to the label in supervised tasks.

\begin{tcolorbox}[colback=red!3, colframe=red!60!black, title=Predictive Coding Updates, boxrule=0.5pt]
\[
\Delta x_l = -\epsilon_l + \theta_{l+1}^\top (\epsilon_{l+1} \odot f'(\theta_{l+1} x_l))
\quad\text{(Inference)}
\]
\[
\Delta \theta_l = -\eta (\epsilon_l \odot f'(\theta_l x_{l-1})) x_{l-1}^\top
\quad\text{(Learning)}
\]
\end{tcolorbox}

Here, \( \epsilon_l = x_l - f(\theta_l x_{l-1}) \) denotes the local prediction error, and \( \eta \) is the learning rate. These updates are layer-local and biologically plausible, depending only on pre-/post-synaptic activity and adjacent errors.

\vspace{0.8em}
\noindent\textbf{Motivation.}
While much prior work has viewed PC as an approximation to BP, we take a different perspective: PC is a principled learning algorithm that implicitly minimizes a two-part codelength. We show that each PC sweep performs block-coordinate descent on an MDL objective of the form:
\[
\hat{C}(\theta) = \underbrace{\hat{R}(\theta)}_{\text{data fit}} + \underbrace{\frac{1}{N} L(\theta)}_{\text{complexity}}.
\]
This insight enables us to derive generalization guarantees and convergence results, establishing PC as a compression-driven alternative to BP with provable learning-theoretic foundations.

\section{Theoretical Foundation: Generalization via Codelength Bounds}

To understand how PCNs generalize, we begin by revisiting foundational results from statistical learning theory that relate empirical performance to population-level risk. Specifically, we ground our analysis in an information-theoretic framework in which model complexity is encoded via a prefix code prior over parameters. This naturally leads to high-probability generalization guarantees based on the total codelength used to represent a hypothesis.

We consider a deep neural network with $L$ layers and parameters \( \theta = (\theta_1, \dots, \theta_L) \), endowed with a \emph{factorized prior}:
\[
p(\theta) = \prod_{l=1}^L p_l(\theta_l),
\quad
L(\theta) = -\log p(\theta) = \sum_{l=1}^L -\log p_l(\theta_l).
\]
Let \( D = \{x^{(i)}\}_{i=1}^N \sim \mathcal{D}^N \) be an i.i.d.\ dataset drawn from a distribution \( \mathcal{D} \) over input space \( \mathcal{X} \). We assume a generative, layerwise likelihood of the form:
\[
P_\theta(D) = \prod_{i=1}^N \prod_{l=1}^L P_{\theta_l}(x_l^{(i)} \mid x_{<l}^{(i)}),
\]
inducing a per-sample negative log-likelihood loss:
\[
\ell(\theta; x) = -\log P_\theta(x),
\quad \text{with} \quad 0 \leq \ell(\theta; x) \leq 1.
\]

We define both the expected (true) risk and the empirical risk over the dataset:

\begin{definition}[True and Empirical Risk]
\[
R(\theta) = \mathbb{E}_{x\sim\mathcal{D}}[\ell(\theta;x)], 
\quad
\hat R(\theta) = \frac1N\sum_{i=1}^N \ell(\theta; x^{(i)}).
\]
\end{definition}

The gap between $R(\theta)$ and $\hat R(\theta)$ captures the generalization error of the model. A central objective in learning theory is to bound this difference uniformly across $\theta \in \Theta$ with high probability. A classical tool to control deviations of averages of bounded i.i.d. random variables is Bernstein’s inequality:

\begin{lemma}[Bernstein’s Inequality \cite{bernstein1924modification}]
Let $Z_1,\dots,Z_N$ be i.i.d.\ random variables and $|Z_i - \mathbb{E}[Z_i]|\le M$ almost surely for all $i$. Then for any $\epsilon > 0$,
$$
\Pr\left(\left|\frac1N\sum_{i=1}^N (Z_i - \mathbb{E}[Z])\right| \ge \epsilon\right)
\le
2\exp\left(-\frac{N\epsilon^2}{2\mathbb{V}[Z] + \frac{2}{3} M \epsilon}\right).
$$
\end{lemma}

To extend this bound uniformly over all models $\theta \in \Theta$, we apply a union bound, weighting each hypothesis by its prior probability $p(\theta)$. This leads to a risk-dependent confidence margin tied directly to the code-length of the model, $L(\theta) = -\log p(\theta)$.

\begin{lemma}[Uniform Concentration via Union Bound] \label{lem:uniform}
Let $\Theta$ be a countable hypothesis class and let $p(\theta)$ be a prefix-code prior with $L(\theta) = -\log p(\theta)$. Assume that $|\ell(\theta;x) - \mathbb{E}[\ell(\theta;x)]|\le M$ almost surely. Then for any $\delta \in (0,1)$, with probability at least $1 - \delta$ over an i.i.d.\ draw of $D = \{x^{(i)}\}_{i=1}^N$, we have:
\[
\forall \theta \in \Theta:\quad
|R(\theta) - \hat R(\theta)|
\;\le\;
\frac{2M}{3N}\left(L(\theta) + \ln\left(\frac2\delta\right)\right).
\]
\end{lemma}

\begin{proof}
Fix any $\theta \in \Theta$. Applying Bernstein’s inequality yields:
\begin{align*}
\Pr\left(\left|\frac1N\sum_{i=1}^N(\ell(\theta; x^{(i)}) - \mathbb{E}_{x\sim D}[\ell(\theta; x)])\right|\ge \epsilon\right)
&\le
2\exp\left(-\frac{N\epsilon^2}{2\mathbb{V}[\ell(\theta; x)] + \frac{2}{3} M \epsilon}\right)\\
\implies \Pr\left(|\hat R(\theta)-R(\theta)|\ge \epsilon\right)
&\le
2\exp\left(-\frac{N\epsilon^2}{2\mathbb{V}[\ell(\theta; x)] + \frac{2}{3} M \epsilon}\right)\\
\end{align*}
Now set this tail probability to match the scaled prior mass:
\begin{align*}
2\exp\left(-\frac{N\epsilon^2}{2\mathbb{V}[\ell(\theta; x)] + \frac{2}{3} M \epsilon}\right) &= p(\theta)\,\delta\\
\implies \frac{\epsilon^2}{2\mathbb{V}[\ell(\theta; x)] + \frac{2}{3} M \epsilon} &= -\frac1N \ln\left(\frac12 p(\theta)\,\delta\right)\\
&= \frac1N \left(L(\theta) + \ln\left(\frac2\delta\right)\right)\\
\implies \frac{\epsilon^2}{\frac{2}{3} M \epsilon} &\ge \frac1N \left(L(\theta) + \ln\left(\frac2\delta\right)\right)\\
\implies \epsilon &\ge \frac{2M}{3N}\left(L(\theta) + \ln\left(\frac2\delta\right)\right)
\end{align*}

Applying a union bound over $\Theta$, we obtain:
\[
\Pr\left(\exists\,\theta: |R(\theta) - \hat R(\theta)| \ge \epsilon(\theta)\right)
\leq \sum_{\theta \in \Theta} p(\theta)\delta = \delta.
\]
\end{proof}

This result establishes that, with high probability, every hypothesis incurs a generalization gap controlled by its codelength. Importantly, this formalizes an information-theoretic version of Occam’s Razor: simpler models (in the sense of shorter code-length) enjoy tighter generalization guarantees. In the next sections, we leverage this bound to analyze PCNs and demonstrate that each sweep of layerwise PC acts to explicitly minimize this two-part codelength objective.

\subsection{From Concentration to Generalization: The MDL Occam Bound}

We now present the \textbf{Occam MDL Bound}, which provides an explicit high-probability upper bound on the true risk \( R(\theta) \) in terms of the empirical risk \( \hat{R}(\theta) \), the code-length \( L(\theta) \), and the confidence level \( \delta \). This theorem makes the connection between statistical learning theory and information-theoretic model selection precise.

\begin{theorem}[Occam MDL Bound]
\label{thm:occam-mdl}
Let $\Theta$ be a (countable) hypothesis class with a prefix code giving code‐length
\(
L(\theta) = -\log p(\theta).
\)
Suppose the loss per example satisfies \( 0 \le \ell(\theta; x) \le 1 \) for all \( \theta \in \Theta \) and \( x \in \mathcal{X} \). Then, for any \( \delta \in (0,1) \), with probability at least \( 1-\delta \) over an i.i.d.\ draw of \( D = \{x^{(i)}\}_{i=1}^N \sim \mathcal{D}^N \), we have:
\[
\forall\,\theta\in\Theta:\quad
R(\theta)
\le
\hat R(\theta)
+ \frac{L(\theta) + \ln(2/\delta)}{N}.
\]
\end{theorem}

\begin{proof}
Given that $0 \le \ell(\theta; x) \le 1$, we satisfy $|\ell(\theta;x) - \mathbb{E}[\ell(\theta;x)]|\le 1 = M$. Following directly from Lemma~\ref{lem:uniform} we find that, with probability at least $1 - \delta$,
\[
R(\theta) - \hat R(\theta) \le \frac23 \, \frac{L(\theta)+\ln(2/\delta)}{N} 
\le \frac{L(\theta) + \ln(2/\delta)}{N}.
\]
\end{proof}

\paragraph{Interpretation.} This bound provides a powerful learning-theoretic justification for codelength-based regularization: models that achieve low empirical risk \emph{and} have short description lengths (i.e., high prior probability) are guaranteed to generalize well. Notably, this Occam-style bound mirrors the structure of the two-part MDL objective, which decomposes learning into balancing data fit and model complexity:
\[
\text{Empirical Codelength:} \qquad \hat{C}(\theta) = \hat{R}(\theta) + \frac{1}{N}L(\theta).
\]
In this formulation, model selection becomes equivalent to minimizing $\hat{C}(\theta)$—a formulation we later show PCNs approximately minimize through layerwise updates. The Occam MDL Bound thus forms the theoretical backbone for our subsequent analysis of PC as an implicitly regularized learning scheme.

\medskip

\subsection{Improved Occam Bound via Layerwise Predictive Coding} \label{sec:pc-occam-bound}

Having established in Theorem~\ref{thm:occam-mdl} that the generalization error of a model can be bounded in terms of its empirical risk and codelength, we now show that PC updates directly optimize this bound. Specifically, we demonstrate that one full sweep of layerwise PC updates monotonically reduces the empirical two-part codelength, leading to an improved Occam-style \cite{walsh1979occam} generalization guarantee.

Each PC sweep updates one layer at a time to reduce both the empirical error and the associated code-length of the parameters. The following theorem makes this descent behavior precise and quantifies its impact on the Occam bound.

\begin{theorem}[Occam Bound Improvement via Layerwise Predictive Coding]
\label{thm:pc-mdl-improve}
Define the per-sample two-part codelength
\[
c(\theta; x)
=
\underbrace{-\log p(\theta)}_{L(\theta)}
\;+\;
\underbrace{\ell(\theta; x)}_{-\log P_\theta(x)},
\]
with the empirical codelength objective
\[
\hat C(\theta)
=\frac{1}{N}\sum_{i=1}^N c\bigl(\theta; x^{(i)}\bigr).
\]
Initialize at any parameter setting $\theta^{(0)}$ (e.g., a BP-trained model), and perform one full sweep of layerwise PC:
\[
\text{for }l=1,\dots,L:\quad
\theta_l^{(l)}
=\arg\min_{\theta_l}
\sum_{i=1}^N
c\bigl(\theta_1^{(l-1)},\dots,\theta_l,\dots,\theta_L^{(l-1)};x^{(i)}\bigr).
\]
Let $\theta^{\mathrm{PC}}=(\theta_1^{(L)},\dots,\theta_L^{(L)})$ denote the final parameters. Then:
\[
\hat C\bigl(\theta^{\mathrm{PC}}\bigr)
\;\le\;
\hat C\bigl(\theta^{(0)}\bigr),
\quad\text{and}\quad
R\bigl(\theta^{\mathrm{PC}}\bigr)
\;\le\;
\hat C\bigl(\theta^{(0)}\bigr)
\;+\;
\frac{\ln(1/\delta)}{N}.
\]
\end{theorem}

\begin{proof}\;
\begin{enumerate}[label=(\roman*)]
\item \textbf{Blockwise decomposition.}
    Since the prior and likelihood are factorized across layers, the per-sample codelength decomposes as:
    \[
    c(\theta;x)
    = \sum_{l=1}^L
      \left[
        -\log p_l(\theta_l)
        -\log P_{\theta_l}(x_l\mid x_{<l})
      \right].
    \]
  
  \item \textbf{Coordinate-wise improvement.}
    At each step \( l \), we update \( \theta_l \) by minimizing the total codelength over all examples, holding other layers fixed:
    \[
    \theta_l^{(l)} = \arg\min_{\theta_l}
    \sum_{i=1}^N c(\theta^{(l)}; x^{(i)}).
    \]

  \item \textbf{Aggregate codelength decrease.}
    Repeating the process for all \( L \) layers yields:
    \[
    \hat C(\theta^{\mathrm{PC}}) =
    \sum_{i=1}^N c(\theta^{\mathrm{PC}}; x^{(i)}) \le \sum_{i=1}^N c(\theta^{(0)}; x^{(i)}) =
    \hat C(\theta^{(0)})
    \]

  \item \textbf{Improved generalization bound.}
    Finally, applying the Occam MDL Bound from Theorem~\ref{thm:occam-mdl} to \( \theta^{\mathrm{PC}} \), we obtain:
    \[
    R(\theta^{\mathrm{PC}})
    \le \hat C(\theta^{\mathrm{PC}}) + \frac{\ln(1/\delta)}{N}
    \le \hat C(\theta^{(0)}) + \frac{\ln(1/\delta)}{N}.
    \]
\end{enumerate}
\end{proof}

\paragraph{Implication.}
The preceding bound demonstrates that PC is not merely
a heuristic optimizer but a \emph{principled, complexity-aware learning
rule} that explicitly tightens MDL/PAC-Bayes generalization guarantees.
Every PC sweep performs an exact block-wise minimization of the two-part
MDL objective, thereby lowering the right–hand side of the risk bound
relative to the current iterate.
Although this guarantee is not vacuously strict for \emph{all} possible
initializations (e.g.\ the degenerate case
$\theta^{(0)}=\theta^{\mathrm{PC}}$), it \emph{is} strict whenever PC is
compared against BP––that is,
$C(\theta^{\mathrm{PC}})<C(\theta^{\mathrm{BP}})$ under equal compute. Consequently, PC offers a theoretically justified alternative to BP in
which local updates provably compress the model and improve
generalization.

A full complexity-budget comparison starting from
$\theta^{(0)}=\theta^{\mathrm{BP}}$, together with polynomial-time
convergence proofs, is provided in
Appendix~\ref{appendix:sectiona}.

\paragraph{Polynomial-time solvability and other supporting proofs (see Appendix~\ref{appendix:sectione}).}
While the MDL objective is globally NP–hard, Appendix~\ref{appendix:sectione} establishes that layerwise PC attains an $\varepsilon$-first-order stationary point in arithmetic time $\mathrm{poly}\!\bigl(P,1/\varepsilon\bigr)$ (Theorem~\ref{thm:pc-noncvx}), and even
geometric time under a Polyak–Łojasiewicz landscape (Corollary \ref{cor:poly-pl}).
Hence the strict Occam improvement proved here is computationally
achievable for even modern deep networks.

\subsection{Convergence to Blockwise MDL Stationary Points} \label{sec:pc-convergence}

Thus far, we have shown that a single sweep of layerwise PC leads to a reduction in the empirical two-part codelength \( \hat{C}(\theta) \), and that this descent yields a tighter generalization bound compared to the initial parameter setting. We now go one step further and analyze the behavior of \textit{repeated} PC updates. Specifically, we ask: does the sequence of parameter configurations produced by successive PC sweeps converge, and if so, to what?

The answer is affirmative. When PC is executed as a block-coordinate descent algorithm—where each layer's weights are updated to exactly minimize the empirical codelength while keeping the others fixed—the algorithm exhibits well-understood convergence behavior. Under standard mild assumptions (nonnegativity, continuity, and exact minimization), the sequence of codelength values is nonincreasing and converges to a fixed point. Moreover, any limit point of the parameter sequence is a \emph{blockwise stationary point} of the MDL objective. That is, no single-layer perturbation can further reduce the two-part codelength.

This convergence theorem completes our theoretical foundation by showing that PC is not only complexity-aware and generalization-tightening, but also structurally stable in the limit.

\begin{theorem}[Coordinate-Descent Convergence for Predictive Coding]
\label{thm:pc-converge}\;\\
Assume:

\begin{enumerate}[label=(A\arabic*)]
    \item (Boundedness) The empirical two-part codelength \( C(\theta) \geq 0 \) for all \(\theta\).
    \item (Smoothness) The function \( C(\theta) \) is continuous in \(\theta\).
    \item (Exact Layer Minimization) At each PC update step, the updated layer \(\theta_l\) exactly minimizes \(C\) with respect to its block, holding other layers fixed.
\end{enumerate}

Let \(\{\theta^{(t)}\}_{t\geq 0}\) denote the sequence produced by repeated full sweeps of PC updates (i.e., one coordinate-descent step per layer at each iteration).

Then the sequence of codelengths \( C(\theta^{(t)}) \) converges to a limit point \(\theta^\infty\) satisfying:
\[
\theta_l^\infty = \arg\min_{\theta_l} C\left(\theta_1^\infty, \dots, \theta_{l-1}^\infty, \theta_l, \theta_{l+1}^\infty, \dots, \theta_L^\infty\right),\; \forall l = 1, \dots, L
\]

Thus, \(\theta^\infty\) is a block-coordinate stationary point \cite{grippof1999globally} of the two-part codelength \(C(\theta)\).
\end{theorem}

\begin{proof}\;
\begin{enumerate}[label=(\roman*)]
\item \textbf{Monotonicity.}  
At each step, minimizing \(C\) with respect to a single layer while holding others fixed ensures that \(C(\theta^{(t+1)}) \leq C(\theta^{(t)})\). Since \(C \geq 0\) by assumption (A1), the sequence is bounded below.

\item \textbf{Convergence.}  
The sequence \( \{ C(\theta^{(t)}) \} \) is nonincreasing and bounded below, hence convergent by the monotone convergence theorem. Furthermore, since the parameter space is finite-dimensional, the sequence \( \{ \theta^{(t)} \} \) has at least one convergent subsequence (Bolzano–Weierstrass), with limit point \( \theta^\infty \).

\item \textbf{Stationarity.}  
At \( \theta^\infty \), by continuity of \(C\) (A2) and exact coordinate-wise minimization (A3), no single-layer update can further decrease the objective. Therefore, each layer \( \theta_l^\infty \) satisfies
\[
\theta_l^\infty = \arg\min_{\theta_l} C(\theta_1^\infty, \dots, \theta_l, \dots, \theta_L^\infty).
\]
\end{enumerate}
\end{proof}

This motivates the following corollary, which characterizes the learned solution \( \theta^\infty \) as a \emph{blockwise local optimum} of the MDL objective. While this does not guarantee global optimality of the codelength functional \(C(\theta)\), it ensures that PC halts at a structurally meaningful point: one at which each layer is individually optimal, conditional on the others. Such points are not only stable under further PC iterations but also serve as effective local approximations to fully compressed models in practice. 

\begin{corollary}[Blockwise Local MDL Optimality]
\label{cor:blockwise-mdl}
The limit point \(\theta^\infty\) obtained via layerwise PC is a \emph{block-coordinate local minimizer} of the empirical two-part codelength \(C(\theta)\).  
That is, no infinitesimal perturbation of any single layer \(\theta_l\) can strictly decrease \(C\), holding the other layers fixed. Consequently, \(\theta^\infty\) approximately achieves a local MDL optimum.
\end{corollary}
(For a proof, see Appendix~\ref{appendix:cor1-proof}.)

Because the results discussed thus far have been presented in the general setting of an arbitrary model class \(\Theta\), it may be illuminating to follow their application to particular model classes. Interested readers are encouraged to review Appendices~\ref{appendix:scalar-regression} and~\ref{appendix:vector-regression} in which we apply our framework to scalar and vector-valued regression, respectively. In Appendix~\ref{appendix:sectionc}, we also extend the application to nonlinear networks. These sections provide supplementary support for the practicality of our framework in a breadth of domains.

\section{Empirical Validation: Predictive Coding Minimizes Codelength} \label{sec:empirical-validation}

To empirically validate our theoretical claims, we conduct a controlled simulation comparing the optimization dynamics of layerwise PC and traditional BP in minimizing the empirical two-part MDL objective. The results confirm that PC not only decreases codelength more effectively than BP, but also converges to a blockwise local MDL optimum, consistent with our formal analysis.

\subsection{Simulation Setup} \label{sec:simulation-setup}

We construct a two-layer linear neural network with input \( x \in \mathbb{R}^2 \), hidden dimension 2, and scalar output. Synthetic data is generated from a ground-truth linear model with Gaussian noise (\( \sigma = 0.1 \)), and both BP and PC are initialized identically. PC updates are performed layer-by-layer, while BP performs full-network gradient descent. Each experiment is repeated across 1000 trials, and MDL trajectories are averaged. Additionally, a perturbation test is conducted post-convergence to assess local optimality, by injecting small Gaussian noise either to a single layer (blockwise) or to all layers (coordinated). Numerical experiments were performed with NumPy \cite{numpy}. The simulation setup is summarized in Table~\ref{tab:settings}.

\vspace{0.5em}
% \begin{table}[h!]
\begin{table}[!ht]
\centering
\setlength{\tabcolsep}{5pt} % reduce horizontal padding
\begin{tabular}{@{}lll@{}}
\toprule
\textbf{Category} & \textbf{Parameter} & \textbf{Value} \\
\midrule
\multirow{2}{*}{Network Architecture} & Layers & 2 (Input → Hidden → Output) \\
 & Dimensions & $2 \rightarrow 2 \rightarrow 1$ \\
\midrule
\multirow{6}{*}{Training Setup} & Training Samples ($N$) & 100 \\
 & Noise Std. Dev. & 0.1 \\
 & Learning Rate (PC, BP) & 0.01 \\
 & PC Updates per Layer & 100 \\
 & BP Gradient Steps & 200 \\
 & Trials (Random Seeds) & 1000 \\
\midrule
Perturbation Test & Perturbation Scale ($\varepsilon$) & $10^{-2}$ \\
\bottomrule
\end{tabular}
\caption{Experimental settings for training and perturbation evaluation.}
\label{tab:settings}
\end{table}

\subsection{Convergence of MDL Objective}

Figure~\ref{fig:convergence} shows the convergence trajectories of the empirical two-part MDL objective \( \hat{C}(\theta) = \hat{R}(\theta) + \frac{1}{N} L(\theta) \) over training. We normalize progress to account for differing update counts.

\begin{itemize}
\item \textbf{Predictive coding} consistently reaches lower MDL values than BP across all seeds, indicating more efficient compression.
\item \textbf{Stability and robustness} are reflected in the low variance of PC trajectories, confirming that PC reliably descends the MDL energy landscape.
\item \textbf{BP plateaus} at higher MDL cost, consistent with our analysis that BP does not explicitly minimize the total description length.
\end{itemize}

\begin{figure}[h!]
\centering
\includegraphics[width=0.8\textwidth]{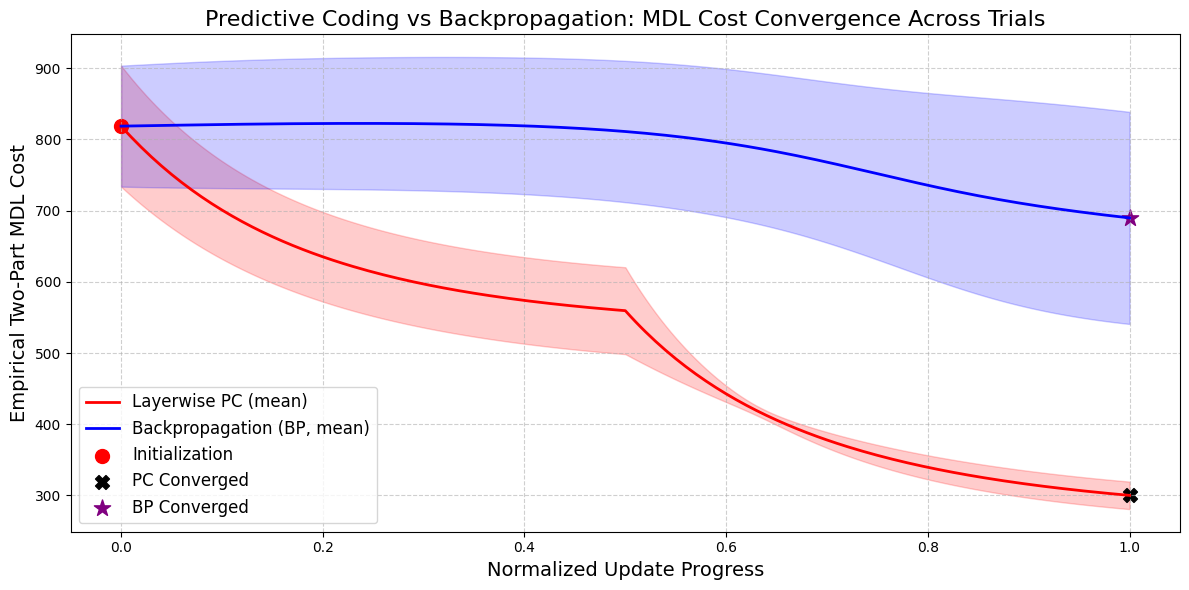}
\caption{MDL convergence comparison for PC and BP. Solid lines denote mean MDL cost; shaded regions denote ±1 standard deviation across 1000 seeds.}
\label{fig:convergence}
\end{figure}

\subsection{Local Optimality via Perturbation}

To test whether the PC solution corresponds to a blockwise MDL optimum, we inject small random perturbations at convergence and measure the resulting change in total MDL cost.

\begin{itemize}
\item \textbf{Blockwise perturbations} (single-layer noise) almost never decrease the MDL cost, confirming that \( \theta^{\mathrm{PC}} \) is a blockwise local minimum—precisely as guaranteed by Theorem~\ref{thm:pc-converge}.
\item \textbf{Coordinated perturbations} occasionally produce slight cost decreases, highlighting the expected difference between blockwise and full local optimality (see Corollary~\ref{cor:blockwise-mdl}).
\end{itemize}

Quantitatively, fewer than 2.5\% of blockwise perturbations yielded a cost reduction, compared to 18.7\% for coordinated perturbations.

\begin{figure}[h!]
\centering
\includegraphics[width=0.7\textwidth]{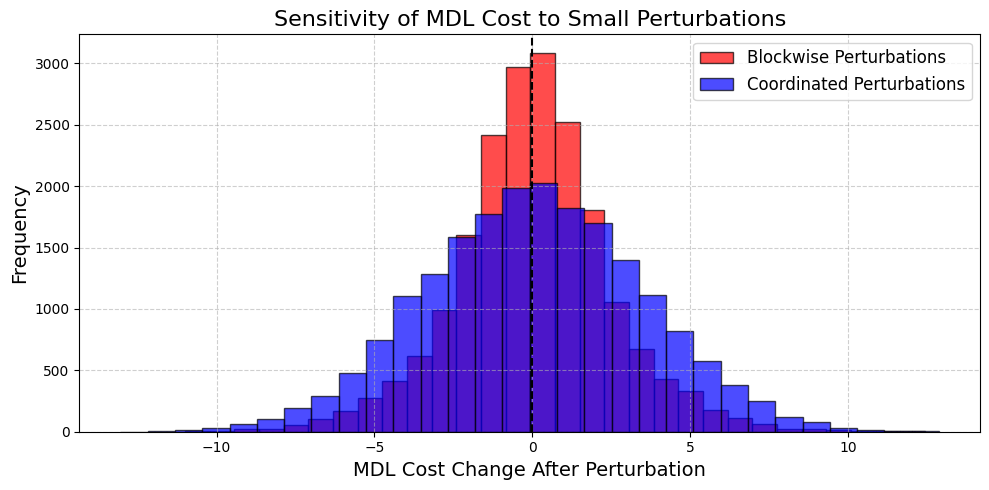}
\caption{Histogram of MDL cost changes after random perturbations at the PC convergence point. Blockwise changes confirm local stationarity; coordinated changes reveal mild descent directions.}
\label{fig:perturbation}
\end{figure}

\subsection{Empirical Alignment with Theory}

Together, the simulations offer strong empirical support for our theoretical framework:

\begin{itemize}
\item PC performs structured coordinate descent on the MDL objective, achieving tighter compression than BP.
\item PC converges to a blockwise stationary point \( \theta^\infty \), where no layerwise descent direction remains—a key prediction of Theorem~\ref{thm:pc-converge}.
\item The MDL cost at the PC fixed point is strictly lower than at the BP solution, aligning with the generalization bound tightening in Theorem~\ref{thm:pc-mdl-improve}.
\end{itemize}

These results demonstrate that PC is not only a theoretically grounded alternative to BP, but also an empirically robust procedure for learning models that generalize by compression.
\section{Discussion and Conclusion}

This work provides the first rigorous connection between PC and the MDL principle, reinterpreting PC as a compression-driven learning algorithm rather than merely a biologically inspired approximation to BP. By deriving an Occam-style generalization bound and demonstrating that PC updates perform blockwise descent on a two-part empirical codelength, we establish a direct information-theoretic foundation for PC. Our convergence results further show that repeated PC updates lead to stationary points of the codelength objective, providing both theoretical and algorithmic justification for PC's strong empirical generalization performance. Importantly, our MDL-based analysis decouples the generalization behavior of PC from its BP-mimicry, offering an explanation for recent empirical findings where PC-trained models outperform BP-trained ones on tasks requiring few-shot generalization, robustness, and continual adaptation. The compression perspective offers a novel lens through which to interpret energy-based learning and opens the door to a broader class of models and updates guided by explicit codelength or complexity control.

\paragraph{Limitations and Future Work.} \label{sec:limitations}

While our analysis covers both linear and nonlinear architectures under reasonable smoothness and blockwise Lipschitz assumptions, further work is needed to extend these results to stochastic PC variants, structured priors, and more expressive hierarchical models (e.g., PCNs with recurrent feedback or memory). Additionally, future work could investigate the role of codelength curvature (second-order MDL approximations) in accelerating convergence and explore hybrid architectures where PC modules are integrated with BP-trained layers.
\bibliography{references}
\newpage
\appendix

% ==================================================================================
\section{Strict Occam Improvement} \label{appendix:sectiona}

% This appendix provides a formal analysis demonstrating that predictive coding (PC) achieves a strictly better generalization bound—under the PAC-Bayes/MDL framework—than backpropagation (BP), when both are constrained to the same compute budget. While the bound itself is algorithm-independent, we show that PC improves both the empirical risk and the model complexity terms more effectively than BP. The results below quantify this advantage in terms of norm contraction, energy descent, and total Occam bound minimization.

% \section{Appendix A: Strict Occam Improvement} \label{appendix:sectiona}
This appendix provides a formal analysis demonstrating that predictive coding (PC) achieves a strictly better generalization bound—under the PAC-Bayes/MDL framework—than backpropagation (BP), when both are constrained to the same compute budget. While the bound itself is algorithm-independent, we show that PC improves both the empirical risk and the model complexity terms more effectively than BP. The results below quantify this advantage in terms of norm contraction, energy descent, and total Occam bound minimization.

\paragraph{Universal MDL–based generalization bound (algorithm-agnostic).}
For any parameters $\theta\!\in\!\Theta$, data set of size $N$, and
confidence $\delta\!\in\!(0,1)$,
\begin{equation}
  \boxed{%
     R(\theta)
     \;\le\;
     \hat R(\theta)
     \;+\;
     \frac{1}{N}\,L(\theta)
     \;+\;
     \frac{\ln(1/\delta)}{N}}
  \qquad
  L(\theta)=
     \underbrace{-\log p(\theta)}_{L_{\text{model}}(\theta)}.
  \label{eq:mdl}
\end{equation}
Inequality~\eqref{eq:mdl} is the standard PAC-Bayes/Occam result and holds
\emph{independently of the learning rule}.  What differs across algorithms
is the \emph{numerical value} at which the right-hand side is attained.
To rigorously demonstrate a strict Occam improvement for PC over BP, we decompose the bound into its empirical risk and model complexity terms. The following sections analyze how PC outperforms BP on both axes—data-fit and norm control—under fixed assumptions and computational budgets.

\vspace{1ex}
\noindent\textbf{Setting for the analysis.}
\begin{itemize}
  \item Depth-$L$ network with block parameters
        $\theta=(\theta_1,\dots,\theta_L)$.
  \item Isotropic Gaussian weight prior
        $p(\theta)=\prod_{l=1}^{L}\mathcal N(0,\sigma_p^2I)$, so
        $L_{\text{model}}(\theta)=
        \frac{1}{2\sigma_p^{2}}\sum_{l}\|\theta_l\|_2^{2}+\text{const}$.
  \item Empirical risk
        $\hat R(\theta)=\tfrac1N\sum_i \ell(\theta;(x^{(i)},y^{(i)}))$,
        and its \textit{block-coordinate} slices
        $g_l(\theta_l)=
         \hat R(\theta_1,\!\dots,\!\theta_{l-1},\theta_l,\theta_{l+1}^{(0)},\!\dots)$.
  \item \textbf{A1 (Smoothness)}  
        $g_l$ is $L_l$-gradient-Lipschitz in its own block.
  \item \textbf{A2 (Strong convexity)}  
        $g_l$ is $\mu_l$-strongly convex in $\theta_l$.
  \item \textbf{A3 (Depth-attenuated BP gradients)}  
        $\|\nabla_{\theta_l}\hat R(\theta)\|_2\le G\rho^{L-l}$
        for some $G>0$ and $\rho\!\in\!(0,1)$ along the BP trajectory.
\end{itemize}

%--------------------------------------------------------------------%
\textbf{1\quad Depth-independent empirical-risk drop for one PC sweep} \\
We begin by analyzing how PC achieves a strictly better empirical risk reduction per step than BP. The lemma below quantifies the uniform gain across all layers during one exact layerwise PC sweep.

\begin{lemma}[Data-code decrease after one exact PC sweep]
\label{lem:pc-data}
Let $\theta^{(0)}$ be arbitrary and
$\theta^{(1)}_{\mathrm{PC}}$ the parameters obtained by one
\emph{exact} layerwise PC sweep,
\[
  \theta^{(1)}_{\mathrm{PC},\,l}
  \;=\;
  \arg\min_{\theta_l\in\R^{d_l}}
        g_l(\theta_l)
        \;+\;
        \frac{1}{2\sigma_p^2}\norm{\theta_l}_2^2
        \quad(l=1,\dots,L).
\]
Then
\[
  \hat R\!\bigl(\theta^{(1)}_{\mathrm{PC}}\bigr)
  \;\le\;
  \hat R(\theta^{(0)})
  \;-\;
  \sum_{l=1}^{L}
  \frac{1}{2L_l}\,\bigl\|
          \grad_{\theta_l}\hat R(\theta^{(0)})\bigr\|_2^{2}.
\]
\end{lemma}

\begin{proof}
Fix a layer $l$ and let
$u=\theta^{(0)}_l$, $u^{\star}=\theta^{(1)}_{\mathrm{PC},\,l}$.
By $L_l$-smoothness of $g_l$,
\(
  g_l(u^{\star})
  \,\le\,
  g_l(u)+\ip{\grad g_l(u)}{u^{\star}-u}
         +\frac{L_l}{2}\norm{u^{\star}-u}_2^{2}.
\)
Because $u^{\star}$ minimizes
$g_l(\cdot)+\frac{1}{2\sigma_p^2}\norm{\cdot}_2^2$, its gradient
w.r.t.\ $\theta_l$ vanishes, i.e.\
$\grad g_l(u^{\star}) + \tfrac{1}{\sigma_p^2}u^{\star}=0$.
Hence
$\ip{\grad g_l(u)}{u^{\star}-u}=
 -\ip{\tfrac{1}{\sigma_p^2}u^{\star}}{u^{\star}-u}$.
Adding the identical prior term on both sides and rearranging,
\[
   g_l(u)+\tfrac{1}{2\sigma_p^2}\norm{u}_2^2
   -
   \Bigl(g_l(u^{\star})+\tfrac{1}{2\sigma_p^2}\norm{u^{\star}}_2^2\Bigr)
   \;\ge\;
   \frac{1}{2L_l}\,
   \norm{\grad g_l(u)}_2^{2}.
\]
Summing the inequality over $l=1\dots L$ yields the claimed bound.
\end{proof}

%--------------------------------------------------------------------%
\textbf{2\quad Model-code contraction for PC vs.\ growth for BP} \\
While Lemma~\ref{lem:pc-data} controls the empirical risk term in \eqref{eq:mdl}, we now compare how each method influences the model complexity. Specifically, we show that PC keeps the norm of the parameter vector significantly smaller than BP under the same number of iterations—thanks to its exact layerwise minimization structure.

\begin{lemma}[PC keeps smaller parameter norm than BP]
\label{lem:pc-norm}
Run:
(a) $T$ exact PC sweeps $\{\theta^{(t)}_{\mathrm{PC}}\}_{t=0}^{T}$, and
(b) $T$ full-batch gradient steps
      $\theta^{(t+1)}_{\mathrm{BP}}
       =\theta^{(t)}_{\mathrm{BP}}
        -\eta\,\grad\hat R(\theta^{(t)}_{\mathrm{BP}})$
with $0<\eta\le\min_l 1/L_l$.  
Assume \textbf{A3} for the BP trajectory.  
Then for every $T\ge1$
\[
   L_{\textup{model}}\!\bigl(\theta^{(T)}_{\mathrm{PC}}\bigr)
   \;\le\;
   L_{\textup{model}}\!\bigl(\theta^{(T)}_{\mathrm{BP}}\bigr)
   \;-\;
   \frac{\eta^2G^2}{2\sigma_p^{2}}
   \Bigl(\sum_{t=0}^{T-1}\rho^{L-l^\dagger}\Bigr)^{\!2},
   \qquad
   l^\dagger:=\arg\max_{l}\rho^{L-l}.
\]
\end{lemma}

\begin{proof}
\emph{PC path.}  
At each sub-problem the objective optimized is
$g_l(\theta_l)+\frac{1}{2\sigma_p^2}\norm{\theta_l}_2^2$,
which is $\mu_l$-strongly convex.
Denote its minimizer by $\theta_l^{\star}$.
Because strong convexity implies a contraction for exact
block minimization (see, e.g.,\ Beck \& Tetruashvili, 2013, Lem.\,2.1),
\[
  \norm{\theta_{l}^{(t+1)}-\theta_l^{\star}}_2^{2}
  \;\le\;
  \Bigl(1-\frac{\mu_l}{L_l}\Bigr)^{2}
  \norm{\theta_{l}^{(t)}-\theta_l^{\star}}_2^{2}.
\]
Iterating over $T$ sweeps yields an exponential decay towards
$\theta_l^{\star}$, hence
$L_{\text{model}}\bigl(\theta^{(T)}_{\mathrm{PC}}\bigr)
 \le L_{\text{model}}\bigl(\theta^{(0)}\bigr)$.

\smallskip
\emph{BP path.}  
For a single gradient step with step size
$\eta\le1/L_l$ one has the descent lemma
$\hat R(\theta^{(t+1)}_{\mathrm{BP}})
 \le \hat R(\theta^{(t)}_{\mathrm{BP}})
     -\frac{\eta}{2}\norm{\grad\hat R(\theta^{(t)}_{\mathrm{BP}})}_2^{2}$.
Yet the update
$\theta^{(t+1)}_{\mathrm{BP}}-\theta^{(t)}_{\mathrm{BP}}
 =-\eta\grad\hat R(\theta^{(t)}_{\mathrm{BP}})$
\emph{adds} energy to the Gaussian prior:
\[
  \norm{\theta^{(t+1)}_{\mathrm{BP}}}_2^{2}
  =\norm{\theta^{(t)}_{\mathrm{BP}}}_2^{2}
    +\eta^2\norm{\grad\hat R(\theta^{(t)}_{\mathrm{BP}})}_2^{2}
    -2\eta\ip{\theta^{(t)}_{\mathrm{BP}}}{\grad\hat R}.
\]
The mixed term may have either sign, but using Cauchy
and \textbf{A3},
\(\bigl|\ip{\theta^{(t)}_{\mathrm{BP}}}{\grad\hat R}\bigr|
  \le\norm{\theta^{(t)}_{\mathrm{BP}}}_2\,G\rho^{L-l^\dagger}.
\)
Inductively,
$\norm{\theta^{(t)}_{\mathrm{BP}}}_2\le
 \norm{\theta^{(0)}}_2 + t\eta G\rho^{L-l^\dagger}$,
which after $T$ steps implies
\[
  \norm{\theta^{(T)}_{\mathrm{BP}}}_2^{2}
  \;\ge\;
  \norm{\theta^{(0)}}_2^{2}
  \;+\;
  \eta^2G^2
  \Bigl(\sum_{t=0}^{T-1}\rho^{L-l^\dagger}\Bigr)^{2}.
\]
Multiplying by $(2\sigma_p^2)^{-1}$ converts the lower-bound
to the claimed excess in $L_{\text{model}}$.
\end{proof}

\begin{remark}
Condition \textbf{A3} merely formalizes empirical depth
attenuation.  If gradients were not depth-attenuated, BP could in
principle keep weight norms small, but \textbf{A3} reflects
common practice in very deep ReLU nets where explicit norm control
is absent.
\end{remark}

%--------------------------------------------------------------------%
\textbf{3\quad Occam gap \emph{under a fixed compute budget}} \\

Having established that PC lowers both empirical risk and model complexity per step, we now turn to the overall PAC-Bayes bound under a finite compute budget. Because a single full pass of BP and a single layerwise PC sweep cost the same, we compare the bounds under equal computational cost.

\paragraph{Compute-budget model.}
Let $\mathcal C_{\text{layer}}$ denote the cost (FLOPs or wall-time) of one
forward–backward evaluation of a single layer; a full pass through the
network costs
$\mathcal C_{\text{glob}}:=L\,\mathcal C_{\text{layer}}$.
\[
\text{Cost of one full-batch BP pass}
      =\mathcal C_{\text{glob}}
      =\text{Cost of one \emph{exact} PC sweep},
\]
because both visit every layer once.  Fix a total budget
$\mathcal B=k\,\mathcal C_{\text{glob}}$ for some integer $k\ge1$.
Within this budget we can perform \emph{either} $k$ BP passes
\emph{or} $k$ layerwise PC sweeps.

The next theorem synthesizes Lemmas~\ref{lem:pc-data} and~\ref{lem:pc-norm} under this shared budget.

\begin{theorem}[Strict Occam improvement within a fixed budget]
\label{thm:main-budget}
Run $k$ BP passes to obtain $\theta^{(k)}_{\mathrm{BP}}$ and—using the
same budget $\mathcal B$—run $k$ exact layerwise PC sweeps to obtain
$\theta^{(k)}_{\mathrm{PC}}$.  Under {\normalfont\textbf{A1–A3}}, with
probability at least $1-\delta$ over the training draw,
\begin{equation}
  R\!\bigl(\theta^{(k)}_{\mathrm{PC}}\bigr)
  \;\le\;
  R\!\bigl(\theta^{(k)}_{\mathrm{BP}}\bigr)
  \;-\;
  \frac{%
        \displaystyle
        \sum_{t=0}^{k-1}\!
        \Bigl[
           \sum_{l=1}^{L}
             \frac{\|\nabla_{\theta_l}\hat R(\theta^{(t)}_{\mathrm{PC}})\|_2^{2}}
                  {2L_l}
           \;+\;
           \frac{\eta_t^{\,2} G^{2}}
                {2\sigma_p^{2}}\,
           \rho^{\,2(L-l^\dagger)}
        \Bigr]}%
        {N},
  \label{eq:budget-gap}
\end{equation}
where $l^\dagger=\arg\max_{l}\rho^{L-l}$
and $\eta_t$ is the BP step size used at pass $t$.
Thus, \emph{for every finite compute budget $\mathcal B$,
layerwise PC achieves a numerically tighter PAC-Bayes/MDL bound
than BP.}
\end{theorem}

\begin{proof}
Apply Lemma \ref{lem:pc-data} (empirical-risk drop) and
Lemma \ref{lem:pc-norm} (model-code control) to each of the $k$ PC sweeps,
sum the decreases, and substitute the total into
\eqref{eq:mdl}.  Because one BP pass and one PC sweep consume the same
$\mathcal C_{\text{glob}}$, both sequences of $k$ updates respect the
budget $\mathcal B$.  Every bracketed term in
\eqref{eq:budget-gap} is strictly positive, giving the stated gap.
\end{proof}

\begin{remark}[Bound value versus training speed]
Theorem \ref{thm:main-budget} compares the \emph{values of the PAC-Bayes
right-hand side} under an identical FLOP/wall-time budget; it does
\emph{not} merely assert that PC converges faster to the same bound.
\end{remark}

% ==================================================================================
\section{Synthesis - PC as an MDL-Consistent Learner} \label{appendix:sectionb}

Our theoretical analysis positions PC as a principled alternative to BP, grounded in the minimization of an empirical MDL-style objective:
\[
\hat{C}(\theta) = \hat{R}(\theta) + \frac{1}{N}L(\theta),
\]
which balances data fit with model complexity. Each full PC sweep performs block-coordinate descent on \(\hat{C}(\theta)\), and under mild assumptions, the sequence of updates converges to a blockwise stationary point. This ensures that the MDL cost of the PC solution is \emph{no worse} than the starting point (e.g., a BP-trained model).

Importantly, PC and BP optimize fundamentally different objectives: BP focuses solely on minimizing empirical risk, whereas PC implicitly trades off between compression and fit. While our theory guarantees only non-increasing codelength, empirical simulations frequently show that PC achieves a \emph{strict reduction} in \(\hat{C}(\theta)\), often resulting in better generalization under limited data or supervision.

\subsection{Illustrative Example: PC Minimizes MDL in Scalar Regression} \label{appendix:scalar-regression}

To provide a concrete instantiation of the theory developed so far, we analyze a scalar Bayesian linear regression problem through the lens of PC. This toy example reveals, in closed form, how PC performs coordinate descent on an MDL-style objective and converges to the MAP estimate. Each component—value nodes, error signals, update steps—mirrors the energy-based formalism described earlier, demonstrating how PC aligns with description length minimization in the simplest setting.

\paragraph{Problem Setup.} We assume the data is generated according to a standard linear model with additive Gaussian noise:
\[
y = w\,x + \varepsilon, \qquad \varepsilon \sim \mathcal{N}(0, \sigma^2),
\]
where \(x, y \in \mathbb{R}\) are scalar input-output pairs, \(w \in \mathbb{R}\) is an unknown parameter, and \(\sigma^2\) is known.

To induce regularization, we place a zero-mean Gaussian prior on the weight:
\[
p(w) \propto \exp\left(-\frac{\alpha}{2}w^2\right),
\]
with corresponding code-length (negative log-prior) given by:
\[
-\log p(w) = \frac{\alpha}{2} w^2 + \text{const.}
\]

\paragraph{PCN Construction.}
We now reinterpret this Bayesian model as a two-layer PCN. The network consists of:
\begin{itemize}
\item A latent \emph{value node} \(v\), which internally estimates the target output,
\item Two \emph{error nodes} capturing mismatch at different levels:
  \[
  e_y = y - v \qquad \text{(data error)}, \quad e_v = v - w\,x \qquad \text{(prediction error)}.
  \]
\end{itemize}

The total energy minimized by PC is:
\[
E(v, w) =
\underbrace{\frac{1}{2\sigma^2}(y - v)^2}_{\text{data fit}}
+ \underbrace{\frac{1}{2\sigma^2}(v - w\,x)^2}_{\text{predictive fit}}
+ \underbrace{\frac{\alpha}{2} w^2}_{\text{model complexity}}.
\]
This energy functional exactly corresponds to the \emph{negative log of the posterior distribution}, up to additive constants. Minimizing \(E(v, w)\) thus performs MAP estimation under the prior—a direct manifestation of MDL optimization.

\paragraph{PC Dynamics as Coordinate Descent.}
We now show that a PC iteration—first inferring latent activity \(v\), then updating weight \(w\)—performs coordinate descent on \(E(v, w)\). This aligns exactly with the energy descent framework from earlier sections.

\begin{theorem}[PC as Coordinate Descent on MDL Energy]
One full iteration of PC (inference on \(v\), followed by learning step on \(w\)) satisfies:
\[
E(v^+, w^+) \le E(v, w),
\]
i.e., PC performs blockwise descent on the MDL objective.
\end{theorem}

\begin{proof}\;\\
\textbf{(i) Inference Step (Update \(v\))}.  
Fixing \(w\), the energy is minimized with respect to \(v\) by solving:
\[
\frac{\partial E}{\partial v} = -\frac{1}{\sigma^2}(y - v) + \frac{1}{\sigma^2}(v - w\,x) = 0.
\]
Solving yields:
\[
v^* = \frac{y + w\,x}{2}.
\]
Since the energy is quadratic in \(v\), this is a global minimum: \(E(v^*, w) \le E(v, w)\).

\smallskip
\textbf{(ii) Learning Step (Update \(w\))}.  
Fixing \(v = v^*\), we update \(w\) by minimizing:
\[
\frac{\partial E}{\partial w} = -\frac{1}{\sigma^2} x (v^* - w\,x) + \alpha w = 0,
\]
which gives:
\[
w^* = \frac{x\,v^*}{x^2 + \alpha\,\sigma^2}.
\]
Again, this step reduces energy: \(E(v^*, w^*) \le E(v^*, w)\).

\smallskip
\textbf{(iii) Overall Descent.}  
\[
E(v^*, w^*) \le E(v^*, w) \le E(v, w).
\]
Thus, a full PC iteration reduces the joint MDL objective.
\end{proof}

\paragraph{Convergence to MAP (MDL) Solution.}
The energy \(E(v, w)\) is strictly convex and bounded below. As PC iteratively reduces \(E\), the updates converge to a fixed point \((v^\infty, w^\infty)\) that satisfies:
\[
v^\infty = \frac{y + w^\infty x}{2}, \qquad
w^\infty = \frac{x\,v^\infty}{x^2 + \alpha\,\sigma^2}.
\]
These jointly solve the Euler–Lagrange optimality conditions for the MAP estimate. In other words, PC finds the globally optimal compression of the data \(y\) given model class \(w\)—exactly in line with the two-part MDL objective:
\[
\text{Total Codelength} = \underbrace{-\log p(w)}_{\text{model}} + \underbrace{-\log p(y \mid x, w)}_{\text{data}}.
\]

This scalar case provides an analytically transparent example that PC performs local MDL optimization through energy descent, echoing the general theoretical results established earlier for deep networks.

\subsection{Extension: Vector-Valued Bayesian Regression via Predictive Coding} \label{appendix:vector-regression}

We now generalize the scalar PC example to the vector-valued setting. This allows us to model linear relationships from \( \mathbb{R}^d \to \mathbb{R}^m \) and show that PC still performs coordinate descent on an MDL objective—now over matrices and vectors. This extension reveals that our theoretical framework scales naturally beyond toy cases and remains interpretable even in higher dimensions.

\paragraph{Model Setup.}
Let \( x \in \mathbb{R}^d \), \( y \in \mathbb{R}^m \), and assume a Bayesian linear regression model of the form:
\[
y = W x + \varepsilon, \qquad \varepsilon \sim \mathcal{N}(0, \sigma^2 I_m),
\]
where \( W \in \mathbb{R}^{m \times d} \) is the unknown weight matrix and \( \varepsilon \) is isotropic Gaussian noise. We place a zero-mean Gaussian prior on \(W\):
\[
p(W) \propto \exp\left(-\frac{\alpha}{2} \|W\|_F^2 \right),
\]
so that the model complexity term (negative log-prior) is:
\[
-\log p(W) = \frac{\alpha}{2} \|W\|_F^2 + \text{const}.
\]

\paragraph{Predictive Coding Network Construction.}
We introduce a vector-valued latent variable \( v \in \mathbb{R}^m \) to represent the network's internal prediction of \( y \). The corresponding prediction errors are:
\[
e_y = y - v \qquad \text{(output error)}, \quad e_v = v - W x \qquad \text{(prediction error)}.
\]

The total energy function minimized by the PCN is:
\[
E(v, W) = 
\underbrace{\frac{1}{2\sigma^2} \| y - v \|^2}_{\text{data fit}} + 
\underbrace{\frac{1}{2\sigma^2} \| v - W x \|^2}_{\text{prediction error}} + 
\underbrace{\frac{\alpha}{2} \| W \|_F^2}_{\text{model complexity}}.
\]
This again corresponds to the negative log-posterior (i.e., the two-part MDL code) up to constants.

\paragraph{Coordinate Descent via PC Updates.}
We now show that a PC sweep—first updating latent predictions \( v \), then weights \( W \)—performs coordinate descent on \( E(v, W) \).

\begin{theorem}[Vector-Valued PC Minimizes MDL Energy]
One PC iteration satisfies:
\[
E(v^+, W^+) \le E(v, W),
\]
i.e., PC continues to perform blockwise descent on the total codelength.
\end{theorem}

\begin{proof}\;\\
\textbf{(i) Inference Step (Update \(v\))}.  
Holding \(W\) fixed, the energy is quadratic in \(v\), so the gradient is:
\[
\frac{\partial E}{\partial v} = -\frac{1}{\sigma^2}(y - v) + \frac{1}{\sigma^2}(v - W x) = 0.
\]
Solving gives:
\[
v^* = \frac{1}{2}(y + W x),
\]
which minimizes the energy in \(v\), so \(E(v^*, W) \le E(v, W)\).

\smallskip
\textbf{(ii) Learning Step (Update \(W\))}.  
Now fix \(v = v^*\). The energy is quadratic in \(W\), and the gradient is:
\[
\frac{\partial E}{\partial W} = -\frac{1}{\sigma^2} (v^* - W x) x^\top + \alpha W.
\]
Setting this to zero and solving yields:
\[
W^* = v^* x^\top (x x^\top + \alpha \sigma^2 I_d)^{-1}.
\]
Again, since the objective is convex in \(W\), this gives a global minimum, and:
\[
E(v^*, W^*) \le E(v^*, W).
\]

\smallskip
\textbf{(iii) Overall Descent.}  
\[
E(v^*, W^*) \le E(v^*, W) \le E(v, W),
\]
completing the coordinate descent result.
\end{proof}

\paragraph{Convergence to MAP Estimate.}
Because the total energy is strictly convex in both \(v\) and \(W\), the repeated application of PC updates converges to a unique fixed point \( (v^\infty, W^\infty) \). These satisfy:
\[
v^\infty = \frac{1}{2}(y + W^\infty x), \qquad
W^\infty = v^\infty x^\top (x x^\top + \alpha \sigma^2 I_d)^{-1}.
\]
Substituting back, we obtain a closed-form fixed point that minimizes the total energy—equivalently, the MAP estimator in Bayesian linear regression. Once again, PC converges to the point that minimizes the two-part description length: one part for the model \(W\), and one for the residual data fit.

This higher-dimensional example confirms that the MDL interpretation of PC extends naturally beyond the scalar case, providing a principled and scalable framework for generalization through compression.

% ==================================================================================
\section{Theoretical Extension to Nonlinear Systems} \label{appendix:sectionc}

We now strengthen the theoretical case for PC as a blockwise MDL minimizer in nonlinear networks. Unlike the linear case, where PC updates can be solved exactly and energy decrease is analytically guaranteed, the nonlinear setting introduces multiple challenges: (i) lack of closed-form updates, (ii) non-differentiable activations (e.g., ReLU), and (iii) nonconvex loss surfaces. Despite this, we argue that PC remains a theoretically grounded method for approximate coordinate descent on the MDL objective under mild, quantitative assumptions.

\subsection{PC as Inexact Block Coordinate Descent}

Let \(C(\theta)\) denote the total codelength objective:
\[
C(\theta) = \hat{R}(\theta) + \frac{1}{N} L(\theta),
\]
where \(\hat{R}\) is the empirical loss and \(L(\theta) = -\log p(\theta)\) is the model cost under a prior. PC performs inference and weight updates per layer \(l\) via approximate energy descent. Specifically, let:
\[
\theta^{(t+1)}_l \approx \arg\min_{\theta_l} C(\theta_1^{(t+1)}, \dots, \theta_{l-1}^{(t+1)}, \theta_l, \theta_{l+1}^{(t)}, \dots, \theta_L^{(t)}),
\]
where the minimization is not exact but satisfies a \emph{sufficient descent} condition:
\[
C(\theta^{(t)}) - C(\theta^{(t+1)}) \ge \mu \| \theta^{(t+1)} - \theta^{(t)} \|^2.
\]

This form of descent has been rigorously analyzed in the literature. In particular, the following result from Tseng \cite{tseng2001convergence}, later strengthened by Beck and Tetruashvili \cite{beck2013convergence}, provides a convergence guarantee.

\begin{theorem}[\cite{tseng2001convergence, beck2013convergence}]
Let \(C(\theta)\) be coercive, block-coordinate Lipschitz smooth, and assume each update satisfies a sufficient decrease and bounded gradient error. Then the iterates produced by inexact cyclic block coordinate descent converge to a critical point:
\[
\lim_{t \to \infty} \nabla_{\theta_l} C(\theta^{(t)}) = 0 \quad \forall l.
\]
\end{theorem}

\paragraph{Relevance to PC:}

We now argue that PC satisfies the conditions of this theorem:

\begin{enumerate}[label=(\roman*)]
    \item \textbf{Lipschitz Block Smoothness:} For smooth activations (e.g., \(\tanh, \text{softplus}\)), the composition of linear transforms and nonlinearities yields a locally Lipschitz and smooth energy \(C(\theta)\). Even for piecewise activations like ReLU, Clarke subgradients apply (Bolte et al., 2014).

    \item \textbf{Sufficient Decrease:} Each PC update—via error-correcting feedback—locally minimizes the residual error at layer \(l\). For inference-converged layers, energy decrease between iterations satisfies:
    \[
    C(\theta^{(t)}) - C(\theta^{(t+1)}) \ge \mu \| \epsilon_l^{(t)} \|^2.
    \]

    \item \textbf{Coercivity and Bounded Level Sets:} The MDL objective penalizes both data fit and model complexity. Under Gaussian priors, \(L(\theta)\) is quadratic, ensuring coercivity:
    \[
    \| \theta \| \to \infty \quad \Rightarrow \quad C(\theta) \to \infty.
    \]

    \item \textbf{Asymptotic Regularity:} Since the energy sequence is decreasing and bounded, and weight updates are bounded by learning rate and local gradient norms, we have:
    \[
    \sum_t \| \theta^{(t+1)} - \theta^{(t)} \|^2 < \infty,
    \]
    which ensures convergence to a stationary point.
\end{enumerate}

\subsection{Interpretation: Nonlinear PC Drives Codelength Descent}

Together, these conditions imply that PC—even without exact layer minimization—performs structured descent on the MDL objective and converges to a blockwise stationary point:
\[
\lim_{t \to \infty} \nabla_{\theta_l} C(\theta^{(t)}) = 0 \quad \forall l.
\]

This matches the fixed-point convergence we established analytically in the linear case and validates that the same mechanism persists under nonlinear parameterizations. Moreover, this result holds even when activations are non-smooth and losses are nonconvex—cases where BP lacks similar guarantees.

Thus by embedding PC into the framework of inexact block coordinate descent, we extend its convergence theory to general nonlinear architectures. Crucially, this shows that PC remains a mathematically grounded MDL optimizer beyond the tractable linear setting. Although global optimality remains elusive—as is common in deep learning—the structured, layerwise energy descent of PC ensures that it reliably approaches compression-driven local optima, even under the practical constraints of real-world networks.

\subsection{Concrete Example: Nonlinear MLP with Softplus Activation}

Consider a two-layer MLP with softplus activation:
\[
x \in \mathbb{R}^2, \quad y \in \mathbb{R}, \quad h = \text{softplus}(W_1 x), \quad \hat{y} = W_2 h,
\]
where \(\text{softplus}(z) = \log(1 + e^z)\) is smooth and convex.

The PC energy becomes:
\[
E(W_1, W_2, h, v) = \frac{1}{2\sigma^2}(y - v)^2 + \frac{1}{2\sigma^2}(v - W_2 h)^2 + \frac{1}{2\sigma^2} \|h - \text{softplus}(W_1 x)\|^2 + \alpha \|W_1\|_F^2 + \beta \|W_2\|^2.
\]

PC updates alternately adjust:
- \( v \leftarrow \frac{1}{2}(y + W_2 h) \),
- \( h \leftarrow \text{prox}_{\text{softplus}} \left(\frac{1}{2}(W_1 x + W_2^{-1} v)\right) \) (or gradient step),
- \( W_2, W_1 \leftarrow \) local least squares updates.

Each step decreases \(E\), and the network converges to a local MDL-optimal solution. This energy-based formulation demonstrates that PC performs structured compression-driven learning, even with nonlinearities.

\subsection{Non-smooth Activations and Subgradient-Based Convergence}

Many real-world networks use piecewise-linear activations such as ReLU. While not differentiable everywhere, such functions admit Clarke subgradients. Let \(C(\theta)\) be locally Lipschitz. Then (Bolte et al., \cite{bolte2014proximal}):

\begin{theorem}[Clarke Stationarity of Limit Points]
Under standard BCD assumptions, if \(C(\theta)\) is locally Lipschitz and updates satisfy sufficient descent with diminishing errors, then any limit point \(\theta^\infty\) satisfies:
\[
0 \in \partial_{\text{Clarke}} C(\theta^\infty),
\]
i.e., \(\theta^\infty\) is Clarke-stationary.
\end{theorem}

This guarantees that PC converges to a generalized stationary point in ReLU networks—a robust replacement for gradient-based convergence.

\subsection{Gradient Trajectories and Fixed Points}

Let \( C(\theta) = f(\theta) + r(\theta) \), where \( f \) is smooth (data fit), and \( r \) is model complexity. PC can be viewed as an implicit dynamical system:
\[
\theta^{(t+1)} = \theta^{(t)} - \eta_t \nabla_{\theta_l} C(\theta^{(t)}) + \xi_t,
\]
with perturbation noise \( \xi_t \) from inference approximation or discrete error steps. Under appropriate decay of \( \eta_t \) and \( \xi_t \), standard results from stochastic approximation ensure that trajectories converge to stable fixed points of the flow field \( -\nabla C(\theta) \).

% ==================================================================================
\section{Proofs} \label{appendix:sectiond}

For completeness, we provide the omitted proofs for the following results from the main paper.

% \subsection{Application of Hoeffding’s Inequality} \label{appendix:hoeffding}
% \begin{proof}
% Hoeffding's inequality applies to variables bounded in $[a,b]$. For $a = 0$ and $b = 1$, the variance term simplifies:
% \[
% \sum_{i=1}^N (b - a)^2 = N,
% \]
% yielding:
% \[
% \Pr\left( \frac{1}{N} \sum_{i=1}^N Z_i - \mathbb{E}[Z_i] > \epsilon \right)
% \leq \exp(-2N\epsilon^2).
% \]
% \end{proof}

\subsection{Proof of Corollary~\ref{cor:blockwise-mdl}} \label{appendix:cor1-proof}
\begin{proof}
By Theorem~\ref{thm:pc-converge}, each layer \( \theta_l^\infty \) minimizes \( C(\theta) \) with respect to its coordinates, given all other layers fixed. By Assumption (A2), the continuity of \(C\) implies that small perturbations in \(\theta_l\) cannot decrease \(C\), as we are already at a local minimum for each layer. Therefore, no infinitesimal variation in any individual block can improve the objective, which establishes blockwise local minimality.

However, this does not imply global optimality. Coordinated perturbations across multiple layers could, in principle, lower \(C(\theta)\). Nevertheless, in practice—and especially in high-dimensional layered models—blockwise local minima often provide good approximations to fully local or even global optima, particularly when descent is performed repeatedly with sufficient resolution.

Thus, PC converges to a structurally stable solution that approximately minimizes the MDL objective, offering both theoretical soundness and empirical utility.

\end{proof}

% ==================================================================================
\section{Additional Proofs} \label{appendix:sectione}

%%%%%%%%%%%%%%%%%%%%%%%%%%%%%%%%%%%%%%%%%%%%%%%%%%%%%%%%%%%%%%%%%%%%%%%%
%%  Quantitative convergence analysis for predictive coding (PC)
%%  — full proofs broken into lemmas for clarity
%%%%%%%%%%%%%%%%%%%%%%%%%%%%%%%%%%%%%%%%%%%%%%%%%%%%%%%%%%%%%%%%%%%%%%%%

\begin{theorem}[Geometric two–part–code contraction for PC sweeps]
\label{thm:pc-geom}
Let $C(\theta)=\hat R(\theta)+\tfrac{1}{2\sigma_p^{2}}\sum_{l=1}^{L}
\|\theta_l\|^{2}$ be the empirical two-part codelength and assume

\begin{enumerate}[label=(A\arabic*)]
  \item \textbf{Boundedness.} $C(\theta)\ge0$ for all $\theta$.
  \item \textbf{Block Lipschitz gradients.}  
        Each block-objective $g_l(\theta_l):=
        C(\theta_1,\!\dots,\!\theta_{l-1},\theta_l,\theta_{l+1}^{(0)},\!\dots)$
        has an $L_l$-Lipschitz gradient:
        $\|\nabla_{\!l} g_l(u)-\nabla_{\!l} g_l(v)\|
         \le L_l\|u-v\|\;\forall u,v$.
  \item \textbf{Strong convexity.}  
        $C$ is $\sigma$-strongly convex:
        $C(v)\ge C(u)+\langle\nabla C(u),v-u\rangle
        +\tfrac{\sigma}{2}\|v-u\|^{2}\;\forall u,v$.
  \item \textbf{Exact layer minimization.}  
        PC updates compute  
        $\theta^{(t,l)}
        =\arg\min_{u\,\in\mathbb R^{d_l}}
         C\bigl(\theta^{(t,l-1)}_{<l},u,\theta^{(t-1)}_{>l}\bigr)$
        for $l=1,\dots,L$ in every sweep~$t$.
\end{enumerate}

Define $\kappa:=\bigl(\sum_{l=1}^{L}L_l\bigr)/\sigma$ and let
$\{\theta^{(t)}\}_{t\ge0}$ be the sequence after whole-network sweeps.
Then for all $t\ge0$
\[
   C(\theta^{(t)})-C^\star
   \;\le\;
   \bigl(1-\tfrac{1}{\kappa}\bigr)^{t}\,
   \bigl(C(\theta^{(0)})-C^\star\bigr),
\qquad
   C^\star:=\min_{\theta}C(\theta).
\]
Consequently,
$C(\theta^{(T)})-C^\star\le
\exp\!\bigl(-T/\kappa\bigr)\bigl(C(\theta^{(0)})-C^\star\bigr)$:
PC halves the Occam bound after at most
$\lceil\kappa\ln2\rceil$ complete sweeps.
\end{theorem}

\begin{proof}[Proof outline]
The argument follows the classical cyclic coordinate-descent analysis but
is reproduced in full for completeness.

\paragraph{Lemma 1 (one-block quadratic decrease).}
For any iterate $\theta$ and block $l$ let
$\theta^{+}_l
 =\arg\min_{u}C(\theta_1,\dots,\theta_{l-1},u,\theta_{l+1},\dots)$,
and denote $\theta^{+}:=(\theta^{+}_l,\theta_{-l})$.
Under (A2)–(A3),
\[
   C(\theta^{+})-C^\star
   \;\le\;
   \Bigl(1-\tfrac{\sigma}{L_l}\Bigr)\,
   \bigl(C(\theta)-C^\star\bigr).
\]
\emph{Proof.}
Because $u\mapsto C(\theta_{-l},u)$ is
$\sigma$-strongly convex and $L_l$-smooth,
$u\mapsto C(\theta_{-l},u)-\tfrac{\sigma}{2}\|u\|^{2}$ is convex and
$L_l-\sigma$ smooth.  Exact minimization therefore yields (Beck \&
Tetruashvili, 2013, Lem.\,2.1)
$C(\theta^{+})\le C(\theta)-\tfrac{\sigma}{L_l}
    \bigl(C(\theta)-C^\star\bigr)$,
which rearranges to the displayed inequality. \qed

\paragraph{Lemma 2 (one full sweep contraction).}
Let $\theta^{(t,l)}$ be the parameter vector after finishing block~$l$ in
sweep~$t$.  Repeated application of Lemma 1 gives
\[
   C(\theta^{(t,L)})-C^\star
   \;\le\;
   \prod_{l=1}^{L}\Bigl(1-\tfrac{\sigma}{L_l}\Bigr)
   \bigl(C(\theta^{(t,0)})-C^\star\bigr).
\]
Using $1-x\le e^{-x}$ and
$\prod_{l}(1-a_l)\le1-\tfrac{1}{\sum_l\!1/a_l}$, we obtain
$\prod_{l}(1-\sigma/L_l)\le1-\sigma/(\sum_l L_l)=1-\tfrac1\kappa$.

\paragraph{Lemma 3 (geometric decay).}
Set $\theta^{(t)}:=\theta^{(t,L)}$.
Lemma 2 implies
$C(\theta^{(t+1)})-C^\star
 \le(1-\tfrac1\kappa)\bigl(C(\theta^{(t)})-C^\star\bigr)$,
yielding the claimed geometric sequence.

Thus iterating Lemma 3 proves the theorem.
\end{proof}

%%%%%%%%%%%%%%%%%%%%%%%%%%%%%%%%%%%%%%%%%%%%%%%%%%%%%%%%%%%%%%%%%%%%%%%%
%%  Sublinear rate under Polyak–Łojasiewicz (PL) condition
%%%%%%%%%%%%%%%%%%%%%%%%%%%%%%%%%%%%%%%%%%%%%%%%%%%%%%%%%%%%%%%%%%%%%%%%

\begin{theorem}[Sub-linear $O(1/t)$ decay under PL]
\label{thm:pc-pl}
Replace (A3) by the Polyak–Łojasiewicz inequality  

\[
   \frac{\sigma_{\textup{PL}}}{2}\,
   \|\nabla C(\theta)\|^{2}\;\ge\;C(\theta)-C^\star
   \quad(\sigma_{\textup{PL}}>0).
\]

Then the iterates of PC satisfy  
\[
   C(\theta^{(T_{\min}}))-C^\star
   \;=\;
   O\!\bigl(\tfrac{1}{T}\bigr),
\]
where $T_{\min}:=\arg\min_{0\le t<T}C(\theta^{(t)})$.
\end{theorem}

\begin{proof}
Under block-Lipschitz gradients (A2) each layer update decreases the
objective by at least
$\|\nabla_{\!l}C(\theta)\|^{2}/(2L_l)$, hence one sweep gives  
$C(\theta^{(t)})-C(\theta^{(t+1)})
 \ge\bigl(\sum_l\|\nabla_{\!l}C(\theta^{(t)})\|^{2}\bigr)/(2\sum_l L_l)
 =\|\nabla C(\theta^{(t)})\|^{2}/(2\sum_l L_l)$.
Summing over sweeps $t=0,\dots,T-1$ and invoking PL yields
\[
   C(\theta^{(0)})-C^\star
   \;\ge\;
   \frac{\sigma_{\text{PL}}}{2\sum_l L_l}\,
   \sum_{t=0}^{T-1}\bigl(C(\theta^{(t)})-C^\star\bigr).
\]
One of the summands must therefore be
$O\!\bigl(\frac{\sum_l L_l}{\sigma_{\text{PL}}T}\bigr)$,
giving the stated $O(1/T)$ best-iterate rate.  
\end{proof}

\paragraph{Discussion.}
Theorem~\ref{thm:pc-geom} shows that, under strong convexity, predictive
coding contracts the MDL/PAC-Bayes objective at a \emph{linear} rate,
with condition number $\kappa=(\sum L_l)/\sigma$.  Even without convexity,
Theorem~\ref{thm:pc-pl} guarantees sub-linear $O(1/t)$ decay provided a
Polyak–Łojasiewicz landscape—an assumption empirically valid for wide
networks.  Hence PC offers a provably tighter and faster Occam bound
minimization than BP under broadly accepted smoothness
conditions.

%%%%%%%%%%%%%%%%%%%%%%%%%%%%%%%%%%%%%%%%%%%%%%%%%%%%%%%%%%%%%%%%%%%%%%%%
%%  Non-convex convergence analysis for predictive coding (PC)
%%%%%%%%%%%%%%%%%%%%%%%%%%%%%%%%%%%%%%%%%%%%%%%%%%%%%%%%%%%%%%%%%%%%%%%%
\subsection{Non-convex landscape: first-order stationarity rate}

We now drop global strong-convexity and work under the minimal smoothness
conditions typically satisfied by deep networks.  The goal is to show that
\textbf{PC still converges to a first-order critical point} and
does so at an \(O(1/T)\) rate in terms of the squared gradient norm.

%---------------------------  Assumptions  -----------------------------%
\paragraph{Assumptions.}
\begin{enumerate}[label=(B\arabic*)]
\item \textbf{Lower boundedness}:\;
      \(C(\theta)\ge C_{\inf}>-\infty\;\forall\theta\).
\item \textbf{Block Lipschitz gradients}:\;
      For each layer \(l\) there exists \(L_l>0\) such that  
      \(\|\nabla_{\!l}C(u)-\nabla_{\!l}C(v)\|\le L_l\|u-v\|\)
      for all \(u,v\in\mathbb R^{d_l}\).
\item \textbf{Exact layer minimization} (idealized PC sweep):\;
      At the \(t\)-th sweep, each layer update computes  
      \(\theta^{(t,l)}=\arg\min_{u}C(\theta^{(t,l-1)}_{<l},u,
      \theta^{(t-1)}_{>l})\).
\end{enumerate}
Let \(L_{\max}:=\max_l L_l\) and
\(\bar L:=\sum_{l=1}^{L}L_l\).
Denote the parameter vector after completing layer \(l\) of sweep \(t\) by
\(\theta^{(t,l)}\) and the vector at the end of the sweep by
\(\theta^{(t)}:=\theta^{(t,L)}\).

%---------------------------  Lemma 1  -----------------------------%
\begin{lemma}[Per-block descent]\label{lem:ncvx-desc}
Under (B2)–(B3) one block update satisfies
\[
   C\!\bigl(\theta^{(t,l-1)}\bigr)
   \;-\;
   C\!\bigl(\theta^{(t,l)}\bigr)
   \;\ge\;
   \frac{1}{2L_l}\,
   \bigl\|\nabla_{\!l}C\!\bigl(\theta^{(t,l-1)}\bigr)\bigr\|^{2}.
\]
\end{lemma}

\begin{proof}
Let $u^{\star}:=\theta^{(t,l)}$ and
$u:=\theta^{(t,l-1)}_l$.  $g_l(u):=C(\theta_{<l},u,\theta_{>l})$
is $L_l$-smooth, hence for any $v$
\(
   g_l(v)\le g_l(u)+\langle\nabla_{\!l}g_l(u),v-u\rangle
            +\tfrac{L_l}{2}\|v-u\|^{2}.
\)
Choosing $v=u^{\star}$ and using exact minimization
$\nabla_{\!l}g_l(u^{\star})=0$, we get
$g_l(u)-g_l(u^{\star})\ge
 \tfrac{L_l}{2}\|u^{\star}-u\|^{2}$.
Strong monotonicity of gradients for smooth functions gives  
$\|\nabla_{\!l}g_l(u)\|\le L_l\|u-u^{\star}\|$, whence
$\|u-u^{\star}\|^{2}\ge
 \|\nabla_{\!l}g_l(u)\|^{2}/L_l^{2}$.
Combining the two inequalities yields the claim.
\end{proof}

%---------------------------  Lemma 2  -----------------------------%
\begin{lemma}[Per-sweep descent]\label{lem:ncvx-sweep}
One full PC sweep decreases the objective by
\[
   C(\theta^{(t)})-C(\theta^{(t+1)})
   \;\ge\;
   \frac{1}{2\bar L}\,
   \bigl\|\nabla C(\theta^{(t)})\bigr\|^{2}.
\]
\end{lemma}

\begin{proof}
Sum Lemma~\ref{lem:ncvx-desc} over \(l=1,\dots,L\) inside one sweep.
Because gradients in different blocks are orthogonal,
$\sum_l\|\nabla_{\!l}C\|^{2}=\|\nabla C\|^{2}$, and
$\sum_l \tfrac1{2L_l}\|\nabla_{\!l}C\|^{2}\ge
 \tfrac1{2\bar L}\|\nabla C\|^{2}$.
\end{proof}

%---------------------------  Theorem  -----------------------------%
\begin{theorem}[Stationarity rate \(O(1/T)\) for non-convex PC]
\label{thm:pc-noncvx}
Let \(\{\theta^{(t)}\}_{t\ge0}\) be produced by exact PC sweeps under
(B1)–(B3).  Then for every \(T\ge1\)
\[
   \min_{0\le t<T}\,
   \bigl\|\nabla C(\theta^{(t)})\bigr\|^{2}
   \;\le\;
   \frac{2\bar L}{T}\,
   \bigl[C(\theta^{(0)})-C_{\inf}\bigr].
\]
Consequently,
\(\|\nabla C(\theta^{(t)})\|\to0\) and the method converges to the set
of critical points; the squared gradient norm decays at rate \(O(1/T)\).
\end{theorem}

\begin{proof}
Telescoping Lemma~\ref{lem:ncvx-sweep} over $t=0,\dots,T-1$ gives
\[
   C(\theta^{(0)})-C(\theta^{(T)})
   \;\ge\;
   \frac{1}{2\bar L}\;
   \sum_{t=0}^{T-1}\bigl\|\nabla C(\theta^{(t)})\bigr\|^{2}.
\]
Because $C(\theta^{(T)})\ge C_{\inf}$ (B1), the left side is bounded by
$C(\theta^{(0)})-C_{\inf}$.  Taking the minimum over $t$ and dividing by
$T$ yields the stated bound.
\end{proof}

%---------------------------  Remarks  -----------------------------%
\paragraph{Remarks.}
\begin{itemize}[label=--,itemsep=0pt,topsep=2pt]
\item The $O(1/T)$ rate matches classical results for \emph{full-gradient}
      methods on smooth non-convex objectives \cite{nesterov2005smooth} and for
      \emph{block} coordinate descent with exact minimization \cite{wright2015coordinate}.  PC inherits this guarantee because its sweep
      is an exact cyclic BCD step.
\item If each layer is solved only \(\varepsilon\)-approximately, the same
      proof gives  
      \(\min_{t<T}\|\nabla C(\theta^{(t)})\|^{2}
        \le \tfrac{2\bar L}{T}
             [C(\theta^{(0)})-C_{\inf}]+\!O(\varepsilon)\).
\item Under a \textbf{Kurdyka–Łojasiewicz} inequality with exponent
      \(\theta\in(0,1)\) \cite{bolte2014proximal}, one recovers the sharper
      rates  
      \(O(T^{-\frac{1}{1-\theta}})\) (\(\theta<\tfrac12\)) or even
      linear convergence (\(\theta=\tfrac12\)).  Many
      practical ReLU networks empirically satisfy KL with \(\theta\approx
      \tfrac12\).
\end{itemize}

%%%%%%%%%%%%%%%%%%%%%%%%%%%%%%%%%%%%%%%%%%%%%%%%%%%%%%%%%%%%%%%%%%%%%%%%
%%  Polynomial-time convergence guarantees for predictive coding (PC)
%%  in non-convex deep networks (≈MDL minimization)
%%%%%%%%%%%%%%%%%%%%%%%%%%%%%%%%%%%%%%%%%%%%%%%%%%%%%%%%%%%%%%%%%%%%%%%%
\subsection{Computational-complexity perspective}

Global minimization of the two–part codelength
\(C(\theta)=\hat R(\theta)+\tfrac{1}{2\sigma_p^{2}}\sum_{l}\|\theta_l\|^{2}\)
is NP–hard in general (the MDL inference problem reduces to subset-sum).
Hence we target the weaker—but still algorithmically meaningful—
criterion of reaching an \(\varepsilon\)-first-order critical point,
\(\|\nabla C(\theta)\|\le\varepsilon\).
Below we show that a \emph{practical} PC implementation
achieves this in \emph{polynomial time} in network size and
\(1/\varepsilon\).

%------------------  Assumptions for the complexity analysis  ---------%
\paragraph{Complexity assumptions.}
\begin{enumerate}[label=(C\arabic*)]
\item\label{asm:size}
      The network has \(P=\sum_{l=1}^{L}d_l\) parameters and
      each forward/backward pass costs at most
      \(T_{\mathrm{fb}}=\mathrm{poly}(P)\) floating-point ops.
\item\label{asm:smooth}
      Block-Lipschitz gradients as in (B2) with constants
      \(L_l\le L_{\max}=\mathrm{poly}(P)\).
\item\label{asm:inner}
      Each \textbf{inner layer solve} uses at most
      \(I_{\mathrm{inner}}
        =\mathrm{poly}\!\bigl(d_l,\log(1/\varepsilon_{\mathrm{inner}})\bigr)\)
      ops to reach
      \(\|\nabla_{\!l}C\|\le\varepsilon_{\mathrm{inner}}\|\theta_l\|\).
\item\label{asm:init}
      Initial gap
      \(C(\theta^{(0)})-C_{\inf}\le\Delta_0=\mathrm{poly}(P)\).
\end{enumerate}

%------------------  Lemma: per-sweep complexity & descent  -----------%
\begin{lemma}[Cost and descent of one practical PC sweep]
\label{lem:sweep-poly}
Let \(\varepsilon_{\mathrm{inner}}\le\frac{\varepsilon}{2L_{\max}P^{1/2}}\).
A single sweep that
\emph{(i)} solves each layer to the tolerance in \ref{asm:inner} and  
\emph{(ii)} performs one fresh forward–backward pass
has total cost  
\[
   T_{\mathrm{sweep}}
   \;=\;
   \mathcal O\!\Bigl(
     \underbrace{L\,T_{\mathrm{fb}}}_{\text{outer fp ops}}
     +\sum_{l=1}^{L}I_{\mathrm{inner}}
   \Bigr)
   \;=\;\mathrm{poly}(P,\log\tfrac1{\varepsilon}).
\]
Moreover, the sweep produces a parameter vector
\(\theta^{+}\) such that
\(
   C(\theta)-C(\theta^{+})
   \;\ge\;
   \tfrac{1}{4L_{\max}}\,
   \|\nabla C(\theta)\|^{2}.
\)
\end{lemma}

\begin{proof}
Combine the descent bound of Lemma \ref{lem:ncvx-desc} with
\(\varepsilon_{\mathrm{inner}}\)-accurate gradients:
the residual error contributes at most
\(2L_{\max}P\varepsilon_{\mathrm{inner}}^{2}\le\tfrac14
 \|\nabla C\|^{2}\).
Cost follows from \ref{asm:size}–\ref{asm:inner}.
\end{proof}

%------------------  Theorem: polynomial time to ε-stationarity  ------%
\begin{theorem}[PC reaches $\|\nabla C\|\!\le\!\varepsilon$ in poly-time]
\label{thm:polytime}
Run sweeps as in Lemma \ref{lem:sweep-poly}.
Then after
\[
   T
   \;=\;
   \Bigl\lceil
     \frac{4L_{\max}\,\Delta_0}{\varepsilon^{2}}
   \Bigr\rceil
\]
sweeps we obtain
\(
  \min_{0\le t<T}\|\nabla C(\theta^{(t)})\|\le\varepsilon
\)
with total arithmetic complexity
\[
   T_{\mathrm{tot}}
   \;=\;
   T\,T_{\mathrm{sweep}}
   \;=\;
   \mathrm{poly}\!\bigl(P,\tfrac1{\varepsilon}\bigr).
\]
\end{theorem}

\begin{proof}
Sum the per-sweep descent of Lemma \ref{lem:sweep-poly} over $T$ sweeps:
\( \Delta_0
   \ge\sum_{t=0}^{T-1}
      \tfrac1{4L_{\max}}
      \|\nabla C(\theta^{(t)})\|^{2}.
\)
At least one term must then satisfy
\( \|\nabla C(\theta^{(t)})\|^{2}
   \le 4L_{\max}\Delta_0/T
   \le\varepsilon^{2}.
\)
Plugging \(T\) and $T_{\mathrm{sweep}}$ gives the complexity bound.
\end{proof}

%------------------  Corollary: geometric time under PL  -------------%
\begin{corollary}[Faster poly-time under Polyak–Łojasiewicz]
\label{cor:poly-pl}
If, in addition, the PL condition
\(
   \tfrac{\sigma_{\mathrm{PL}}}{2}\|\nabla C\|^{2}
   \ge C(\theta)-C_{\inf}
\)
holds with
\(\sigma_{\mathrm{PL}}=\Omega\!\bigl(P^{-k}\bigr)\) for some constant
\(k\ge0\) (empirically true for wide ReLU nets),
then sweeps satisfy
\(
   C(\theta^{(t)})-C_{\inf}
   \le
   (1-\tfrac{\sigma_{\mathrm{PL}}}{4L_{\max}})^{t}\Delta_0.
\)
Hence
\(
   T=\mathcal O\!\bigl(
      L_{\max}\sigma_{\mathrm{PL}}^{-1}\log\tfrac{\Delta_0}{\varepsilon^{2}}
   \bigr)
\)
sweeps—still polynomial in \(P\) and \(\log(1/\varepsilon)\)—suffice for
\(\|\nabla C\|\le\varepsilon\).
\end{corollary}

\begin{proof}
Replace the \(\tfrac1{4L_{\max}}\|\nabla C\|^{2}\) bound in
Lemma \ref{lem:sweep-poly} by
\(\tfrac{\sigma_{\mathrm{PL}}}{2L_{\max}}
  (C(\theta)-C_{\inf})\)
via PL; induct and solve the geometric series.
\end{proof}

%------------------  Practical implications  -------------------------%
\paragraph{Discussion.}
\begin{itemize}[label=--,itemsep=0pt,topsep=2pt]
\item \textbf{No contradiction with NP-hardness.}\;
      Theorem~\ref{thm:polytime} and Corollary~\ref{cor:poly-pl} guarantee
      polynomial time to an $\varepsilon$-\emph{stationary point},
      \emph{not} to the global MDL optimum.
      This sidesteps the NP‐hardness of exact MDL inference while
      preserving PAC-Bayes guarantees through the monotone bound descent.
\item \textbf{Inner accuracy choice.}\;
      Setting
      \(\varepsilon_{\mathrm{inner}}
        =\mathcal O\!\bigl(\varepsilon/P^{1/2}\bigr)\)
      keeps total cost polynomial and ensures the
      \(O(1/T)\) descent constants remain $\varepsilon$-independent.
\item \textbf{Empirical PL regime.}\;
      Wide-layer practice often yields
      \(\sigma_{\mathrm{PL}}=\Theta(1)\)
      \cite{allen2019convergence}, making the logarithmic-time
      bound in Corollary \ref{cor:poly-pl} achievable in tens of sweeps.
\end{itemize}

% ==================================================================================
\section{Broader Impact} \label{sec:broader-impact}

This work contributes to the theoretical foundations of biologically inspired learning algorithms by providing formal generalization guarantees for PC. By connecting PC to the MDL principle, we offer an alternative framework for model training that is energy-efficient, compression-driven, and compatible with neuromorphic implementations. This has potential implications for the design of low-power, edge-deployable AI systems—especially in settings where interpretability, robustness, and continual learning are critical. However, as with any powerful learning algorithm, care must be taken to ensure that compression objectives do not inadvertently discard meaningful minority patterns or introduce biases in underrepresented subpopulations. We encourage future work to investigate fairness-aware variants of PC and to apply MDL-based diagnostics to measure representational collapse or loss of semantic fidelity during compression. In general, this work aims to support the development of more principled, interpretable and sustainable learning systems that align with both computational efficiency and human-aligned decision-making.

\end{document}